\documentclass[runningheads]{llncs}
\usepackage{graphicx}
\usepackage{subfig} 
\usepackage{amsfonts}
\usepackage{amssymb}
\usepackage{amsmath,bm}

\usepackage{mathrsfs}

\usepackage{scalerel}

\usepackage[T1]{fontenc}

\usepackage{enumitem}
\setlist{noitemsep} 
\setlist{nosep} 

\usepackage{etoolbox}
\newcommand{\zerodisplayskips}{%
  \setlength{\abovedisplayskip}{0pt}%
  \setlength{\belowdisplayskip}{0pt}%
  \setlength{\abovedisplayshortskip}{0pt}%
  \setlength{\belowdisplayshortskip}{0pt}}
\appto{\normalsize}{\zerodisplayskips}
\appto{\small}{\zerodisplayskips}
\appto{\footnotesize}{\zerodisplayskips}

\BeforeBeginEnvironment{figure}{\vskip-3ex}
\AfterEndEnvironment{figure}{\vskip-3ex}

\newcommand\scale[2]{\vstretch{#1}{\hstretch{#1}{#2}}}
	
\newcommand{\LIPplus}{\if@draft
		\mathbin{\ooalign{$\bigtriangleup$\crcr\hidewidth \raise.14em\hbox{$\scale{0.7}{\scriptscriptstyle+}$}\hidewidth}}
	\else
		\mathbin{\mathpalette\LIPcls+}
	\fi}
\newcommand{\LIPminus}{\if@draft
		\mathbin{\ooalign{$\bigtriangleup$\crcr\hidewidth \raise.14em\hbox{$\scale{0.7}{\scriptscriptstyle-}$}\hidewidth}}
	\else
		\mathbin{\mathpalette\LIPcls-}
	\fi}
\newcommand{\LIPtimes}{\if@draft
	  \mathbin{\ooalign{$\bigtriangleup$\crcr\hidewidth \raise.14em\hbox{$\scale{0.7}{\scriptscriptstyle\times}$}\hidewidth}}
	\else
		\mathbin{\mathpalette\LIPcls\times}
	\fi}

\newcommand{\LIPcls}[2]{%
  \ooalign{$#1\bigtriangleup$\crcr  \hidewidth\raisefix{#1}\hbox{$#1\scale{0.45}{\bm{#2}}$}\hidewidth}}

\def\raisefix#1{%
  \ifx#1\displaystyle
    \raise.14em
  \else
    \ifx#1\textstyle
      \raise.14em
    \else
      \ifx#1\scriptstyle
        \raise.112em
      \else
        \raise.0933em
      \fi
    \fi
  \fi
}

\newcommand{\LP}{\LIPplus}
\newcommand{\LM}{\LIPminus}
\newcommand{\LT}{\LIPtimes}

\newcommand{\Real}{\mathbb R}
\newcommand{\Realb}{\overline{\Real}}

\newcommand{\la}{\lambda}

\newcommand{\I}{\mathcal{I}}
\newcommand{\Ib}{\overline{\I}}
\newcommand{\Fcurv}{\mathcal{F}}
\newcommand{\Fcurvb}{\overline{\Fcurv}}

\newcommand{\Lscr}{\mathscr{L}}

\begin{document}
\title{\MakeUppercase{Logarithmic Mathematical Morphology: a new framework adaptive to illumination changes}}
\titlerunning{Logarithmic Mathematical Morphology}
%
\author{Guillaume Noyel\inst{1,2}\orcidID{0000-0002-7374-548X}}
\authorrunning{G. Noyel}
%
\institute{University of Strathclyde Institute of Global Public Health, Ecully, France \and
International Prevention Research Institute, iPRI, Lyon, France}
\maketitle              
\begin{abstract}
A new set of mathematical morphology (MM) operators adaptive to illumination changes caused by variation of exposure time or light intensity is defined thanks to the Logarithmic Image Processing (LIP) model. This model based on the physics of acquisition is consistent with human vision. The fundamental operators, the logarithmic-dilation and the logarithmic-erosion, are defined with the LIP-addition of a structuring function. The combination of these two adjunct operators gives morphological filters, namely the logarithmic-opening and closing, useful for pattern recognition. The mathematical relation existing between ``classical'' dilation and erosion and their logarithmic-versions is established facilitating their implementation. Results on simulated and real images show that logarithmic-MM is more efficient on low-contrasted information than ``classical'' MM. 

\end{abstract}

%
%

\section{Introduction}
\label{sec:intro}

Images are functions whose values are bounded between 0 and $M$ (e.g. 256 for 8 bits images) and depend of the illumination conditions. During the acquisition, some parts of the image may be underexposed to light with dark values close to 0, whereas other parts may be overexposed to light with bright values close to $M$. Therefore, the processing should not be the same in the bright and in the dark parts. When processing grey-level images by Mathematical Morphology (MM) \cite{Matheron1967,Serra1982,Heijmans1994}, dark parts and bright parts are processed in the same way and in some cases the transformed image may have values that exceed the upper limit $M$.

The aim of this paper is to overcome this issue by defining morphological operators adaptive to lighting variations thanks to an appropriate model, the Logarithmic Image Processing (LIP) one \cite{Jourlin2001,Jourlin2016} which allows to brighten or darken images in a way compatible with the physics of acquisition and with the human visual system \cite{Brailean1991}.  
Previously, a morphological transform, the LIP-top-hat, was defined with the LIP model in \cite{Jourlin2001}. A model, the Symmetric Logarithmic Image Processing (SLIP) one \cite{Navarro2013} was combined with wavelets in \cite{Navarro2014}. Homomorphic models \cite{Oppenheim1968,Navarro2013} and retinex algorithms \cite{Meylan2006} were also used with convolution.
However, these models, interesting from a mathematical point of view, are not related to the physics of acquisition.

This paper constitutes the first attempt to define morphological operators adaptive to lighting variations - without any pre-processing. Such a property makes the definition of Logarithmic-Mathematical Morphology of the utmost importance for many applications where the acquisition depends on the illumination (e.g. industry, outdoor scenes, forensics, medical images etc.) \cite{Sugimura2015,Lai2017,Noyel2017c,Peng2017}. The paper is organised as follows: i) after a reminder about MM and the LIP model, ii) MM will be defined in the logarithmic-additive framework. The morphological properties of the operators will be verified and iii) the Logarithmic-MM will be illustrated and compared to the classical MM.

%
%

\section{Prerequisites}
\label{sec:pre}


\subsection{Mathematical Morphology}
\label{ssec:pre:MM}

\begin{definition}{\textbf{Complete lattice.}}
Given a set $\Lscr$ and a partial order $\leq$ on $\Lscr$, $\Lscr$ is a complete lattice if every subset $\mathscr{X}$ of $\Lscr$ has an infimum (a greatest lower bound), $\wedge \mathscr{X}$, and a supremum (a least upper bound), $\vee \mathscr{X}$.
\end{definition}
MM \cite{Matheron1967,Serra1982,Najman2013} is defined on complete lattices \cite{Heijmans1994,Banon1993}.
The least element $O$ and the greatest element $I$ are two important elements of the lattice $\Lscr$.
A grey-level image is a function $f$: $D \subset \Real^n \rightarrow [O,M[$, with $M \in \Real$. The space of images is denoted $\I = [O,M[^D$.
The (bounded) set of images $\Ib = [0,M]^D$ and the set of functions $\Realb^D$, $\Realb = \Real \cup \{-\infty , +\infty\}$ are complete lattices with the order $\leq$. For $\Ib$, the least and greatest elements are the constant functions equal to zero, $f_0$, and $M$, $f_M$. The supremum and infimum are respectively, for any $\mathscr{X} \subset \overline{\I}$: $\left(\wedge_{\overline{\I}} \mathscr{X}\right)(x) = \wedge_{[0,M]} \left\{ f(x):f \in \mathscr{X}, \> x\in D \right\}$ and $\left(\vee_{\overline{\I}} \mathscr{X}\right)(x) = \vee_{[0,M]} \left\{ f(x):f \in \mathscr{X}, \> x\in D  \right\}$.

\begin{definition}{\textbf{Erosion, dilation} \cite{Serra1988,Banon1993}\textbf{.}}
Given $\Lscr_1$ and $\Lscr_2$ two complete lattices, a mapping $\psi \in \Lscr_2^{\Lscr_1}$ is: i) an erosion $\varepsilon$: iff $\forall \mathscr{X} \subset \Lscr_1$, $\psi( \wedge \mathscr{X} ) = \wedge \psi( \mathscr{X} )$ or ii) a dilation $\delta$: iff $\forall \mathscr{X} \subset \Lscr_1$, $\psi( \vee \mathscr{X} ) = \vee \psi( \mathscr{X} )$.
\label{pre:def_morpho_base}
\end{definition}
As the definitions of these mappings apply even to the empty subset of $\Lscr_1$, we have: $\varepsilon(I)=I$ and $\delta(O)=O$.

\begin{definition}{\textbf{Adjunction} \cite{Heijmans1994}\textbf{.}}
Let $\varepsilon \in \Lscr_2^{\Lscr_1}$ and $\delta \in \Lscr_1^{\Lscr_2}$ be operators between complete lattices $\Lscr_1$ and $\Lscr_2$; the pair $(\varepsilon,\delta)$ is called an adjunction between $\Lscr_1$ and $\Lscr_2$ if for all $X \in \Lscr_1$, $Y \in \Lscr_2$ there is $\delta(Y) \leq X \Leftrightarrow Y \leq \varepsilon(X)$.
\label{pre:def_adjunction}
\end{definition}

\begin{proposition}
In an adjunction $(\varepsilon,\delta)$, $\varepsilon$ is an erosion and $\delta$ a dilation \cite{Heijmans1994}.
\end{proposition}

When using an additive structuring function $b \in [O,M]^{D_b}$, $D_b \subset D$, invariant under translation (in $D$), the previously defined dilation $\delta$ and erosion $\varepsilon$ in the  lattice $(\Realb^D,\leq)$, can be expressed as \cite{Serra1988,Heijmans1994}:
\begin{eqnarray}
\delta_b(f)(x)		 =\vee 	\left\{ f(x - h) + b(h), h \in D_b\right\} = (f \oplus b) (x) \label{eq:dilate_funct}\\
\varepsilon_b(f)(x)=\wedge \left\{ f(x + h) - b(h), h \in D_b\right\} = (f \ominus b) (x). \label{eq:erode_funct}%
\end{eqnarray}
$\oplus$ and $\ominus$ are the extension to functions of Minkowski operations on sets \cite{Serra1982}.

\begin{definition}{\textbf{Opening, closing} \cite{Serra1982,Heijmans1994}\textbf{.}}
An operator $\psi \in \Lscr^{\Lscr}$ on the complete lattice $\Lscr$ is called:
\begin{itemize}[wide, labelwidth=0pt, labelindent=0pt]
	\item an opening if $\psi$ is increasing ($\forall X, Y \in \Lscr$, if $X \leq Y$ then $\psi(X) \leq \psi(Y)$), anti-extensive ($\forall X \in \Lscr$, $\psi(X) \leq X$) and idempotent ($\psi \circ \psi = \psi$),
	\item a closing if $\psi$ is increasing, extensive ($\forall X \in \Lscr$, $X \leq \psi(X)$) and idempotent.
\end{itemize}
\label{pre:def_opening_closing}
\end{definition}

\begin{proposition}
Let $(\varepsilon,\delta) \in \Lscr_2^{\Lscr_1} \times \Lscr_1^{\Lscr_2}$ be an adjunction between $\Lscr_1$ and $\Lscr_2$, then $\delta \varepsilon$ is an opening on $\Lscr_1$ and $\varepsilon \delta$ is a closing on $\Lscr_2$ \cite{Heijmans1994}.
\end{proposition}



\subsection{Logarithmic Image Processing}
\label{ssec:pre:LIP}

The LIP model, introduced by Jourlin et al. \cite{Jourlin2001,Jourlin2016}, is a mathematical framework for image processing based on the physical law of transmittances. Its consistency with the human visual model \cite{Brailean1991} makes it suitable not only for images acquired with transmitted light but also for images acquired with reflected light. Due to the relation, $T_f = 1 - f/M$, between the transmittance $T_f(x)$ at point $x$ and the grey level $f(x)$, the grey scale is inverted in the LIP-model: 0 corresponds to the white extremity, when no obstacle is placed between the light source and the sensor, whereas $M$ corresponds to the black value, when the source cannot be transmitted through the obstacle. 
The addition of two images corresponds to the superposition of two obstacles generating the images $f$ and $g$: 
\begin{equation}
	f \LIPplus g = f + g - fg/M. \label{eq:LIP:plus}%
\end{equation}
	The multiplication of an image $f$ by a real number $\la$ is deduced from the equation \ref{eq:LIP:plus}, $\lambda \LIPtimes f = M - M \left( 1 - f/M \right)^{\lambda}$, and corresponds to a variation of thickness (or opacity) of the object.
The opposite function $\LM f$ of $f$ and the difference between two grey level functions $f$ and $g$ are expressed by:
\begin{align}
	\LM f = (-f)/(1-f/M) & \qquad \text{and} &	f \LM g = (f-g)(1-g/M). \label{eq:LIP:minus}%
\end{align}

Let us note that $\LM f$ is not an image (as it takes negative values) and $f \LM g$ is an image if and only if $f \geq g$.
\begin{property}[\textbf{\textit{A strong physical property}}]
The negative values $\LM f$, with $f \geq 0$, are light intensifiers that can be used to compensate the light attenuation due a variation of exposure-time (or light intensity) \cite{Jourlin2016}.
\end{property}
\begin{property}[\textbf{\textit{Mathematical properties}}]
Let $\Fcurv = ]-\infty,M[^D$ be the space of functions with values in $]-\infty,M[$. The space $(\Fcurv,\LP,\LT)$ is a real \textit{vector space} and $(\I,\LP,\LT)$ represents the positive cone of this vector space \cite{Jourlin2001,Jourlin2016}. $\Fcurv$ and $\I$ are both ordered by the usual order $\leq$ \cite{Jourlin2001}.
\end{property}

%
%

\section{Logarithmic-Mathematical Morphology}
\label{sec:LIP_MM}

MM is defined on the lattice of functions with real values $\Realb^D$. When performing a dilation by a (non-flat) structuring function, the dilated function may have values which exceed the range $[0,M[$. In order to solve this issue and to perform operations adaptive to light variations, let us extend MM with Logarithmic-MM, on the lattice of functions $\Fcurvb=[-\infty,M]^D$ with values in $[-\infty,M]$. First of all, the fundamental operators of erosion and dilation are needed:

\begin{proposition}
Let $f \in \Fcurvb$ be a function and $b \in ]-\infty,M[^{D_b}$ a structuring function, the mappings in $\Fcurvb$ defined by
\begin{align}
\delta_b^{\LP}(f)(x)		 &=\vee 	 \left\{ f(x - h) \LP b(h), h \in D_b\right\} \label{eq:logarithmic-dilation}\\
\varepsilon_b^{\LP}(f)(x)&=\wedge \left\{ f(x + h) \LM b(h), h \in D_b\right\} \label{eq:logarithmic-erosion}%
\end{align}
form an adjunction $(\varepsilon_b^{\LP} , \delta_b^{\LP})$, with $\varepsilon_b^{\LP}$ an erosion and $\delta_b^{\LP}$ a dilation.
$\varepsilon_b^{\LP}$ is called a logarithmic-erosion and $\delta_b^{\LP}$ a logarithmic-dilation.
\label{prop:LMM:base_operators}
\end{proposition}

\begin{proof}{of proposition \ref{prop:LMM:base_operators}.}
\begin{itemize}[wide, labelwidth=0pt, labelindent=0pt]
\item $\delta_b^{\LP}$ is a dilation (def. \ref{pre:def_morpho_base}). As the operation $\LP$ preserves the order $\leq$ \cite{Jourlin2001}, we have $\forall f, g \in \Fcurvb$, $\forall x \in D$,
$\delta_b^{\LP} ( f \vee g )(x)	= \vee_{h \in D_b} \left\{ \left( (f \vee g)(x-h) \right) \LP b(h) \right\} 
	= \vee_{h \in D_b} \left\{ \left( f(x-h) \LP b(h) \right) \vee \left( g(x-h) \LP b(h) \right) \right\}
	= \left[ \vee_{h \in D_b} \left\{ f(x-h) \LP b(h)\right\} \right] \vee$\\ $\left[ \vee_{h \in D_b} \left\{ g(x-h) \LP b(h)\right\} \right]
	= \delta_b^{\LP} ( f )(x) \vee \delta_b^{\LP} ( g )(x)$.

In addition, with $b(h)\in ]-\infty,M[$, we have:
$\delta_b^{\LP} (O)(x) = \delta_b^{\LP}(f_{-\infty})(x) = \vee_{h \in D_b} \{ ( -\infty(x-h) \LP b(h) ) \}
							= \vee_{h \in D_b} \{ -\infty(1-b(h)/M) + b(h) \} = -\infty
							= O(x)$.


\item Similarly, $\forall f, g \in \Fcurvb$, $\varepsilon_b^{\LP} ( f \wedge g ) = \varepsilon_b^{\LP} ( f ) \wedge \varepsilon_b^{\LP} ( g )$ and $\varepsilon_b^{\LP} (I) = \varepsilon_b^{\LP}(f_M)\nonumber = M = I.$ Therefore, $\varepsilon_b^{\LP}$ is an erosion. 

\item $(\varepsilon_b^{\LP} , \delta_b^{\LP})$ is an adjunction because:
$\delta_b^{\LP} (f) \leq g \Leftrightarrow \forall x \in D,\> \vee_{h \in D_b} \{ f(x-h) \LP b(h) \} \leq g(x)
	\Leftrightarrow \forall x \in D, \forall h,\> f(x-h) \LP b(h) \leq g(x)
	\Leftrightarrow \forall y \in D, \forall h,\> f(y) \leq g(y+h) \LM b(h)
	\Leftrightarrow \forall y \in D,\> f(y) \leq \wedge_{h \in D_b} \{g(y+h) \LM b(h)\}
	\Leftrightarrow f \leq \varepsilon_b^{\LP}(g)$.
\end{itemize}
\end{proof}

Let us verify that the new operations are dual by their \textsl{negative function}.

\begin{proposition} 
The \textsl{negative function} \cite{Heijmans1994} is $f^*(x) = \LM f(x)$.
\label{prop:negative_function}
\end{proposition}

\begin{proposition} 
Let $\overline{b}(x) = b(-x)$ be the reflection of the structuring function $b$, the logarithmic-erosion $\varepsilon_{b}^{\LP}$ and dilation $\delta_b^{\LP}$ are dual by their \textsl{negative function}:
\begin{align}
	(\delta_b^{\LP} (f^*))^* = \varepsilon_{\overline{b}}^{\LP} (f) \quad \text{and} \quad (\varepsilon_{b}^{\LP} (f^*))^* = \delta_{\overline{b}}^{\LP} (f). \label{eq:duality_LMM}
\end{align}
\label{prop:LMM_duality}
\end{proposition}

\begin{proof}
$(\delta_b^{\LP} (f^*))^*(x) = \LM (\vee_{h \in D_b} \{ \LM f(x-h) \LP b(h)\})
	= \wedge_{h \in D_b} \{f(x-h) \LM b(h)\}
	= \wedge_{h \in D_{\overline{b}}} \{f(x+h) \LM \overline{b}(h)\} = \varepsilon_{\overline{b}}^{\LP} (f)(x)$.
	Similarly, we have $(\varepsilon_{b}^{\LP} (f^*))^* = \delta_{\overline{b}}^{\LP} (f)$.
\end{proof}

As $(\varepsilon_b^{\LP} , \delta_b^{\LP})$ is an adjunction, an opening and a closing can be defined \cite{Heijmans1994}.
\begin{proposition}
Given an adjunction $(\varepsilon_b^{\LP} , \delta_b^{\LP})$, the operator $\gamma_{b}^{\LP} = \delta_b^{\LP} \varepsilon_b^{\LP}$ is an opening and $\varphi_{b}^{\LP} = \varepsilon_b^{\LP} \delta_b^{\LP}$ is a closing (by adjunction). $\gamma_{b}^{\LP}$ is called a logarithmic-opening and $\varphi_{b}^{\LP}$ a logarithmic-closing.
\label{def:LMM:opening_closing}
\end{proposition}


A relation between the logarithmic operations and the ``classical '' ones exists. This facilitates the implementation of the logarithmic operations, as the ``classical'' ones are available in many image analysis toolboxes.

\begin{proposition}
	Let $f \in \overline{\Fcurv}$ be a function and $b \in ]-\infty,M[^{D_b}$ a structuring function, the logarithmic-dilation $\delta_b^{\LP}$ and the logarithmic-erosion $\varepsilon_b^{\LP}$ are related to the dilation $\delta_b$ and erosion $\varepsilon_b$ by:
	\begin{align}
		\delta_b^{\LP} f = M (1 - \exp{(- \delta_{\acute{b}}(\acute{f}) )}) & \qquad\text{and} & \varepsilon_b^{\LP} f = M (1 - \exp{(-\varepsilon_{\acute{b}}(\acute{f})   )})\label{eq:LIP_dil_ero_prop}%
	\end{align}
	with $\acute{f} = -\ln{\left( 1 - f/M \right)}$, $\acute{f} \in \Realb$.
	\label{prop:link_MM_LMM}
\end{proposition}

\begin{proof}
The dilation $\delta_b$ and the erosion $\varepsilon_b$ are mappings of the lattice $\Realb^D$, whereas the logarithmic-dilation $\delta_b^{\LP}$ and erosion $\varepsilon_b^{\LP}$ are mappings of the lattice $[-\infty,M]^D$. In order to link these operations, a bijective mapping (i.e. an isomorphism) is needed between these two lattices. Such an isomorphism $\xi: \Realb^D \rightarrow [-\infty,M]^D$ and its inverse $\xi^{-1}$ are both defined in \cite{Jourlin1995} by $[\xi(f)](x) = M(1-\exp{(-f/M)})$ and $\left[\xi^{-1}(f)\right](x) = -M \ln{(1-f/M)}$. As increasing bijections, $\xi$ ans $\xi^{-1}$  distribute over infima and suprema. Therefore, the dilation $\delta_b^{\LP}$ can be expressed by:
$\delta_b^{\LP}f(x) = \xi \circ \xi^{-1} (\vee_{h \in D_b} \left\{ f(x-h) \LP b(h) \right\})
	= \vee_{h \in D_b} \{ M(1 - e^{\ln{\left(1 - \frac{f(x-h) \LP b(h)}{M} \right)} } )\}
	= M(1 - e^{ \wedge_{h \in D_b} \ln{\left( (1-\frac{f(x-h)}{M}) (1-\frac{b(h)}{M}) \right)} })
	= M(1 - e^{-\vee_{h \in D_b} \left\{ -\ln{ \left(1-\frac{f(x-h)}{M}\right) } - \ln{\left(1-\frac{b(h)}{M}\right)} \right\} })
	= M(1 - e^{-\delta_{\acute{b}} (\acute{f}) })$.
Similarly, we have $\varepsilon_b^{\LP} = M(1 - e^{-\varepsilon_{\acute{b}}(\acute{f}) })$.
\end{proof}

%
%

\section{Results}
\label{sec:res}

\begin{figure}[htb]
	\centering
	\begin{tabular}{c@{\hspace{0.8cm}}c}
  \includegraphics[width=0.30\columnwidth]{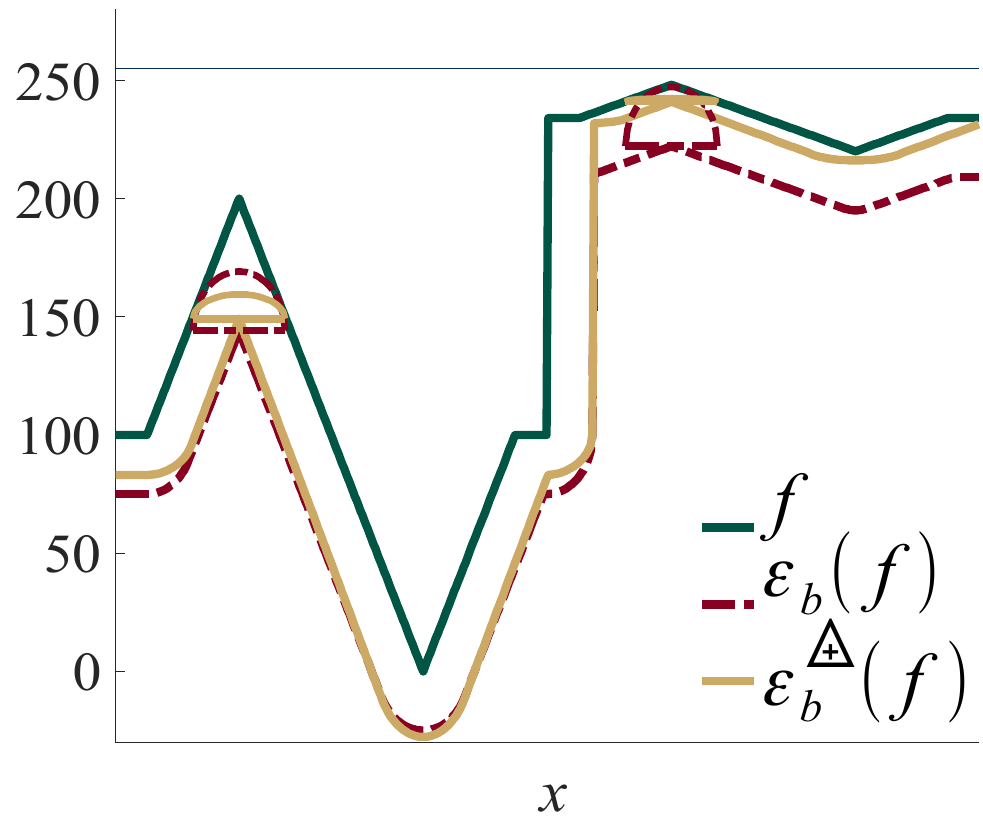}&
	\includegraphics[width=0.30\columnwidth]{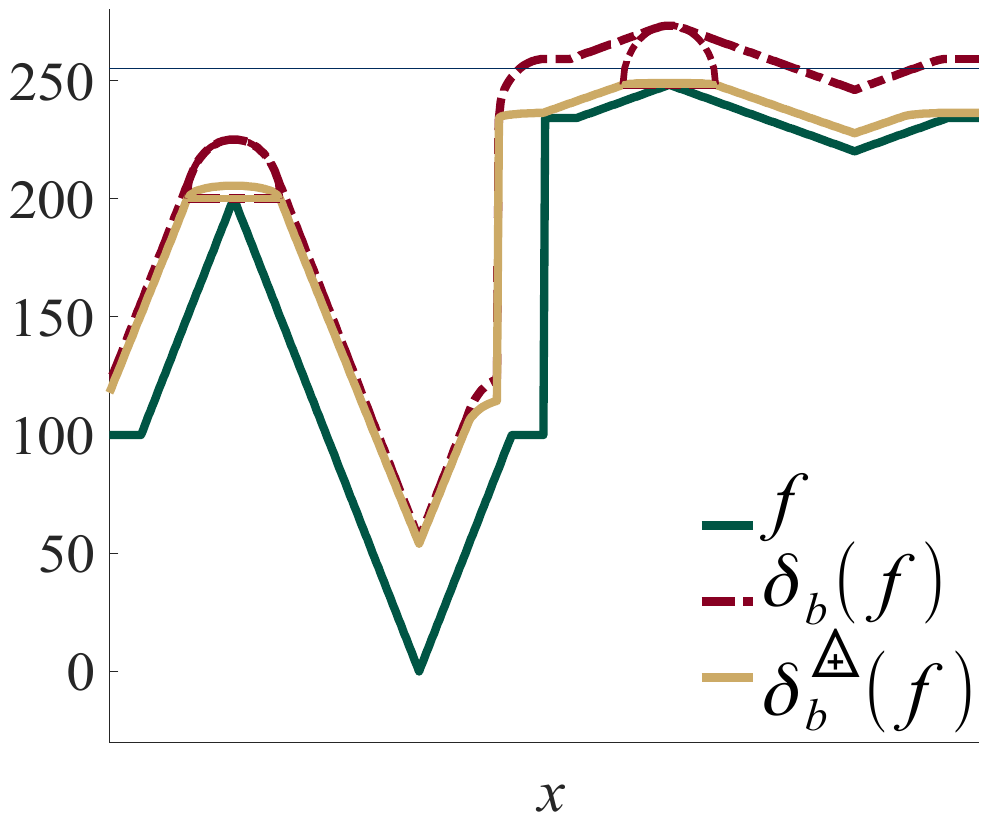}\\
	(a) Erosions &
	(b) Dilations\\
	\includegraphics[width=0.30\columnwidth]{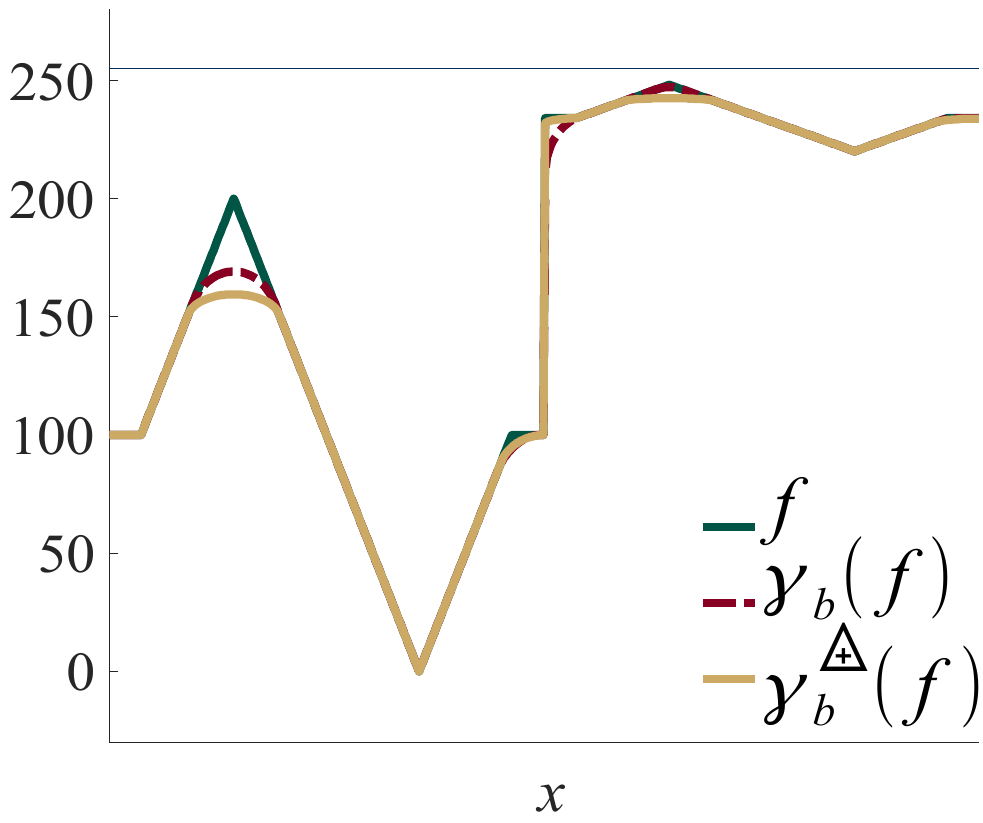}&
	\includegraphics[width=0.30\columnwidth]{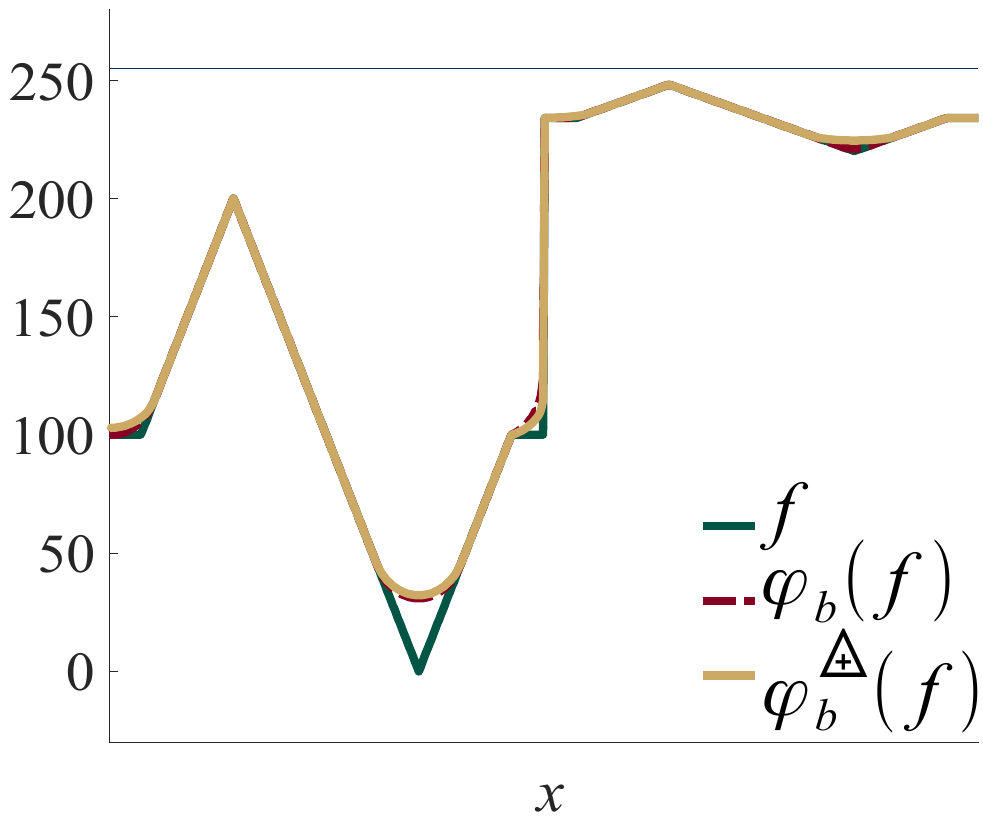}\\
	(c) Openings&	
	(d) Closings
	\end{tabular}
\caption{Comparison between the operations of ``classical'' MM and Logarithmic-MM: (a) erosions $\varepsilon_b(f)$, $\varepsilon_b^{\protect \LP}(f)$, (b) dilations $\delta_b(f)$, $\delta_b^{\protect \LP}(f)$, (c) openings $\gamma_b(f)$, $\gamma_b^{\protect \LP}(f)$ and (d) closings $\varphi_b(f)$, $\varphi_b^{\protect \LP}(f)$. (a), (b) The structuring function is represented for both peaks and for all the operations $\varepsilon_b(f)$, $\varepsilon_b^{\protect \LP}(f)$, $\delta_b(f)$ and $\delta_b^{\protect \LP}(f)$.}
\label{fig:comp_MM_vs_LMM_signal}
\end{figure}

In the figure \ref{fig:comp_MM_vs_LMM_signal}, the ``classical'' operations of MM and those of Logarithmic-MM are compared on a simulated signal. For each operation of Logarithmic-MM, the amplitude of the structuring function (sf) changes according to the values of the image because of the LIP-laws, $\LP$ or $\LM$, used in their definition (eq. \ref{eq:logarithmic-dilation} and \ref{eq:logarithmic-erosion}), whereas for the operations of ``classical'' MM the amplitude of the sf remains the same. In the figure \ref{fig:comp_MM_vs_LMM_signal}.b, the values of the logarithmic-dilation $\delta_b^{\LP}(f)$ always remain below the upper limit $M = 256$, whereas for the ``classical'' dilation $\delta_b(f)$, some grey-levels may exceed the limit $M$. Such a property is due to the LIP addition law $\LP$. In the figure \ref{fig:comp_MM_vs_LMM_signal}.a, the lowest values of both erosions $\varepsilon_b^{\LP}(f)$ and $\varepsilon_b(f)$ are negative because of the laws $\LM$  and $-$ used in equations \ref{eq:logarithmic-erosion} and \ref{eq:erode_funct}. For real value functions $\Realb^D$, the negative values have no physical justification, whereas in the LIP model they correspond to light intensifiers \cite{Jourlin2016}. In the figure \ref{fig:comp_MM_vs_LMM_signal}.c (resp. \ref{fig:comp_MM_vs_LMM_signal}.d), the disparity between the openings $\gamma_b (f)$ and $\gamma_b^{\LP} (f)$ (resp. closings $\varphi_b(f)$ and $\varphi_b^{\LP}(f)$) is greater for the grey-levels close to $M$ than for those close to zero. Indeed, the LIP difference $\LM$ is non linear along the grey-level axis.

In the figure \ref{fig:comp_MM_vs_LMM_Nessie}, two images of the same scene are acquired at two different exposure time (i.e. shutter speed): a bright image $f$ (Fig. \ref{fig:comp_MM_vs_LMM_Nessie}.a) and a dark image $f^d$ (Fig. \ref{fig:comp_MM_vs_LMM_Nessie}.d). Both images $f$ and $f^d$ are complemented  ($f^c = M-1-f$) before computing a morphological gradient $\varrho_b(f^c) = \delta_b(f^c) - \varepsilon_b(f^c)$ \cite{Serra1988} (Fig. \ref{fig:comp_MM_vs_LMM_Nessie}.b, \ref{fig:comp_MM_vs_LMM_Nessie}.e) and its logarithmic version $\varrho_b^{\LP}(f^c) = \delta_b^{\LP}(f^c) \LM \varepsilon_b^{\LP}(f^c)$ (Fig. \ref{fig:comp_MM_vs_LMM_Nessie}.c, \ref{fig:comp_MM_vs_LMM_Nessie}.f). For comparison purpose, the amplitudes of each gradient are scaled between 0 and 255. The logarithmic-gradient of the dark image $\varrho_b^{\LP}((f^d)^c)$ (Fig. \ref{fig:comp_MM_vs_LMM_Nessie}.f) finds much more contours than the ``classical'' one $\varrho_b((f^d)^c)$ (Fig. \ref{fig:comp_MM_vs_LMM_Nessie}.e). Even on the bright image $f$, the logarithmic-gradient $\varrho_b^{\LP}(f^c)$ (Fig. \ref{fig:comp_MM_vs_LMM_Nessie}.c) finds more contours than the ``classical'' one  $\varrho_b^{\LP}(f)$, especially on the darkest parts (Fig. \ref{fig:comp_MM_vs_LMM_Nessie}.b). The logarithmic-gradient is also few sensitive to lighting variations (Fig. \ref{fig:comp_MM_vs_LMM_Nessie}.c and f).

\begin{figure}[htb]
	\centering
	\begin{tabular}{@{}c@{\hspace{0.8cm}}c@{\hspace{0.8cm}}c@{}}
  \includegraphics[width=0.29\columnwidth]{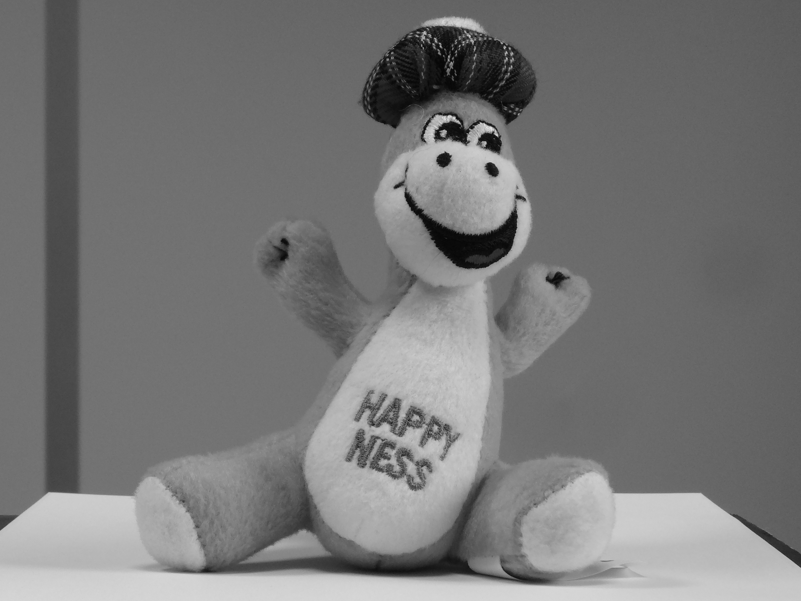}&
	\includegraphics[width=0.29\columnwidth]{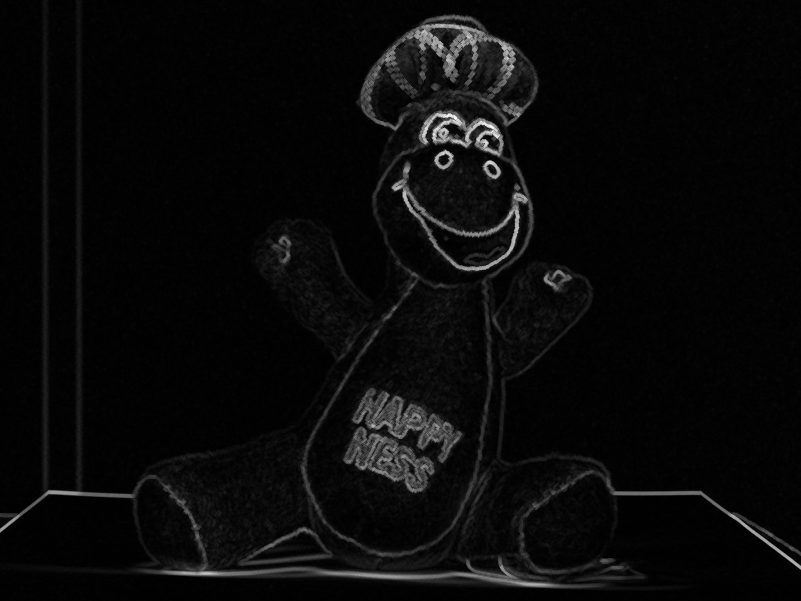}&
	\includegraphics[width=0.29\columnwidth]{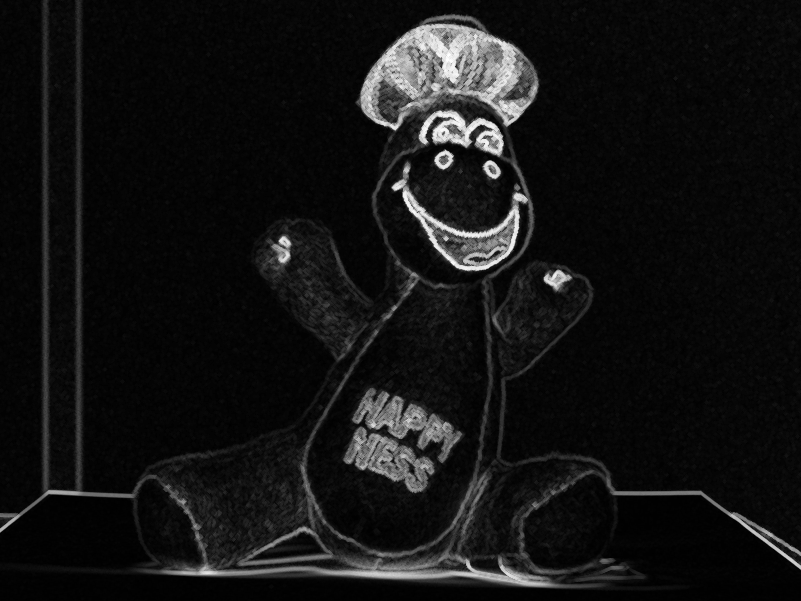}\\
	(a) Bright image $f$&
	(b) Gradient $\varrho_b(f^c)$&
	(c) L-gradient $\varrho_b^{\LP}(f^c)$\\
	\includegraphics[width=0.29\columnwidth]{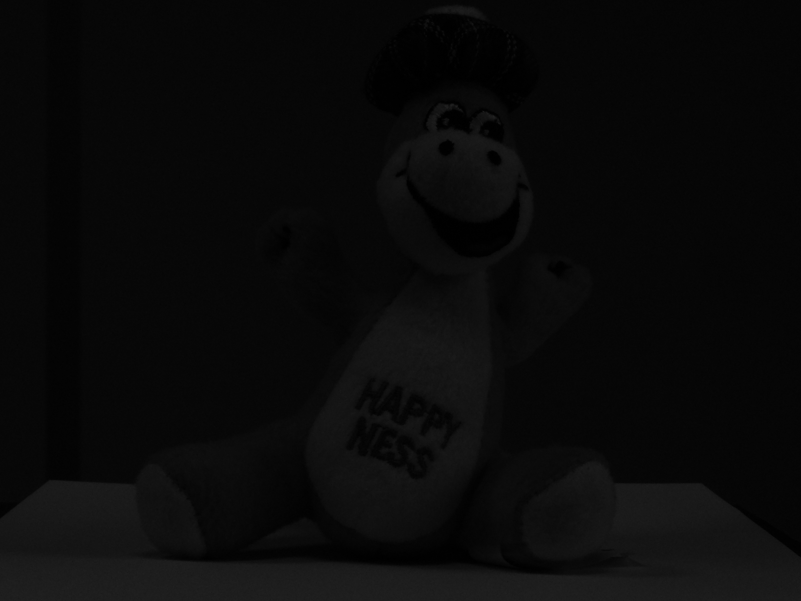}&
	\includegraphics[width=0.29\columnwidth]{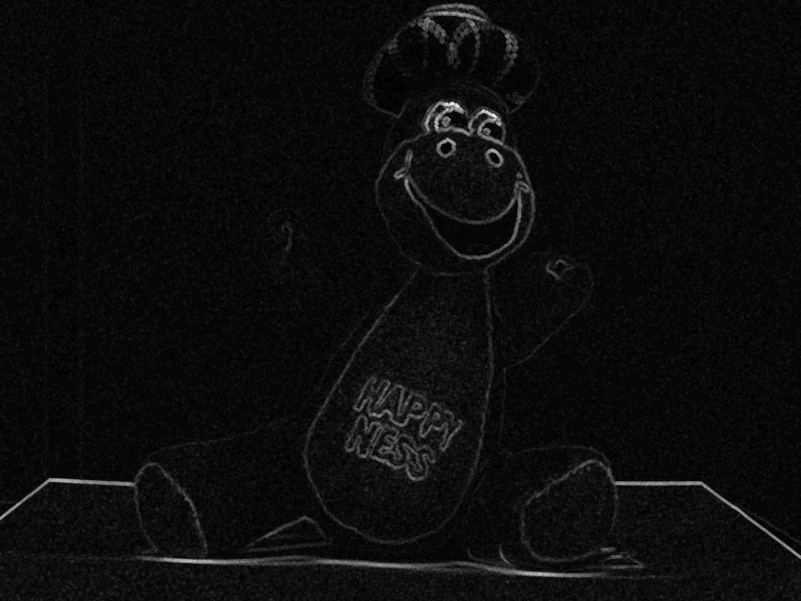}&
	\includegraphics[width=0.29\columnwidth]{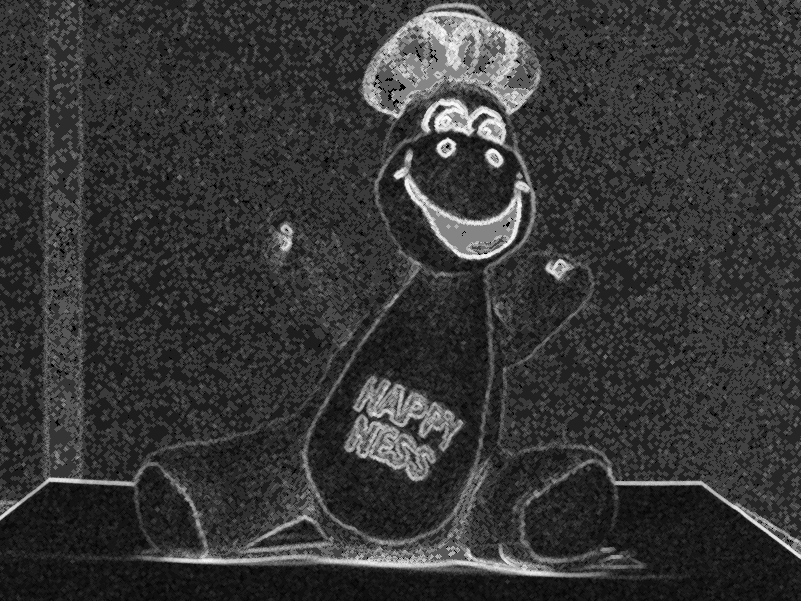}\\
	(d) Dark image $f^d$&
	(e) Gradient $\varrho_b((f^d)^c)$&
	(f) L-gradient $\varrho_b^{\LP}((f^d)^c)$
	\end{tabular}
\caption{Comparison between the gradient $\varrho_b$ (b, e) and the Logarithmic-gradient $\varrho_b^{\protect \LP}$ (c, f) on a bright image $f$ (a) (acquired with an exposure time of 1/40 s) and on a dark image of the same scene (d) (exposure time of 1/800 s). The sf $b$ is an hemisphere of radius 2 pixels.}
\label{fig:comp_MM_vs_LMM_Nessie}
\end{figure}

In the figure \ref{fig:comp_open_close_vs_Lopen_Lclose_Nessie}, an opening $(\gamma_b((f^d)^c))^c$, a closing $(\varphi_b((f^d)^c))^c$, a logarithmic-opening $(\gamma_b^{\LP}((f^d)^c))^c$ and a logarithmic-closing $(\varphi_b^{\LP}((f^d)^c))^c$ are compared on the complement of the dark image $f^d$ (Fig. \ref{fig:comp_MM_vs_LMM_Nessie}.b), using an hemisphere of radius 15 pixels as structuring function. For comparison purpose, the amplitudes of each filtered image are scaled between 0 and 255. The ``classical'' opening and closing (Fig. \ref{fig:comp_open_close_vs_Lopen_Lclose_Nessie}.a, \ref{fig:comp_open_close_vs_Lopen_Lclose_Nessie}.c) have a limited effect in terms of transformation whereas the logarithmic-opening and closing have a more important effect (Fig. \ref{fig:comp_open_close_vs_Lopen_Lclose_Nessie}.b, \ref{fig:comp_open_close_vs_Lopen_Lclose_Nessie}.d). E.g. on the body of the soft toy monster, the words are removed with the logarithmic-opening and closing and not with the ``classical'' opening and closing.

Those examples illustrate the property of the logarithmic-MM to adapt to important contrast variations caused by varying illumination conditions.

\begin{figure}[htb]
	\centering
  \begin{tabular}{c@{\hspace{0.8cm}}c}
  \includegraphics[width=0.29\columnwidth]{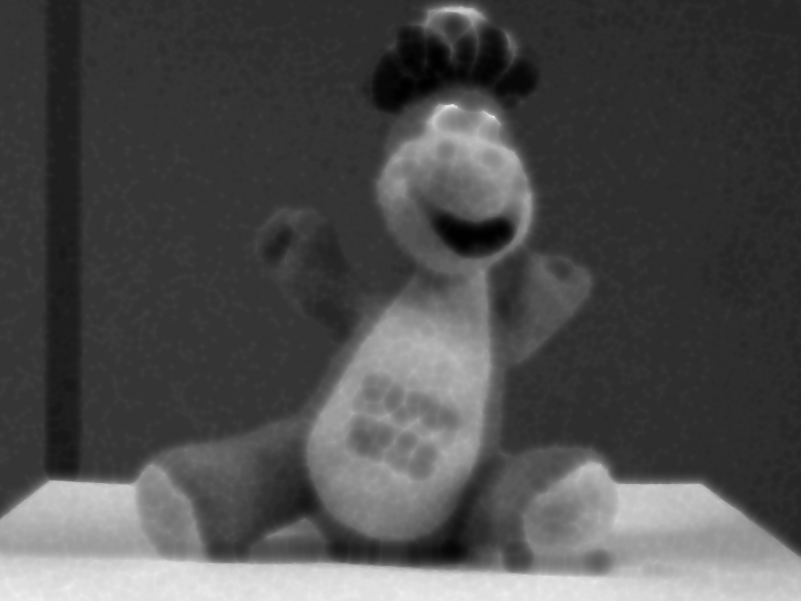}&
	\includegraphics[width=0.29\columnwidth]{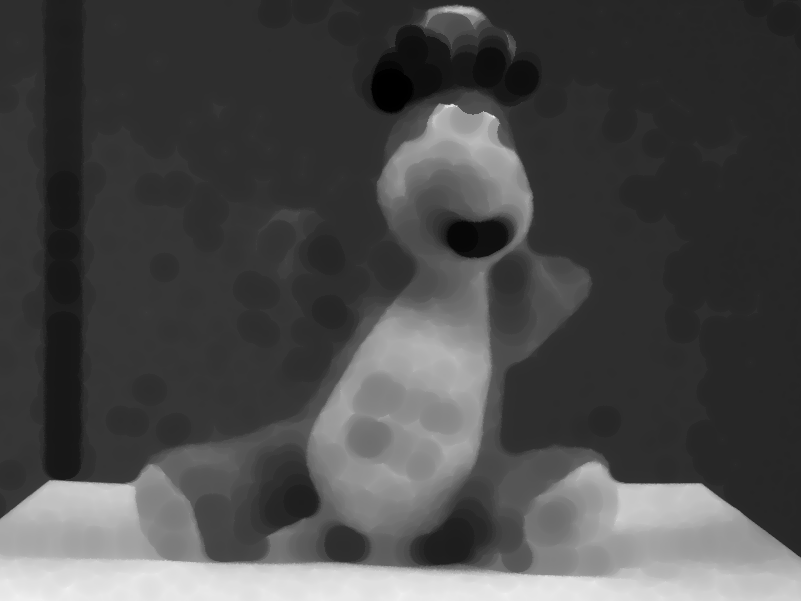}\\
	(a) Opening&
	(b) L-opening\\
	\includegraphics[width=0.29\columnwidth]{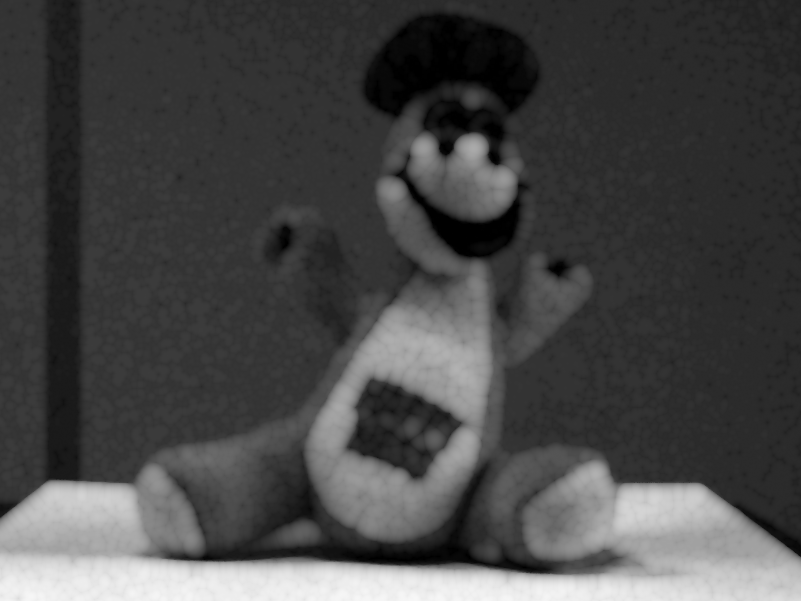}&
	\includegraphics[width=0.29\columnwidth]{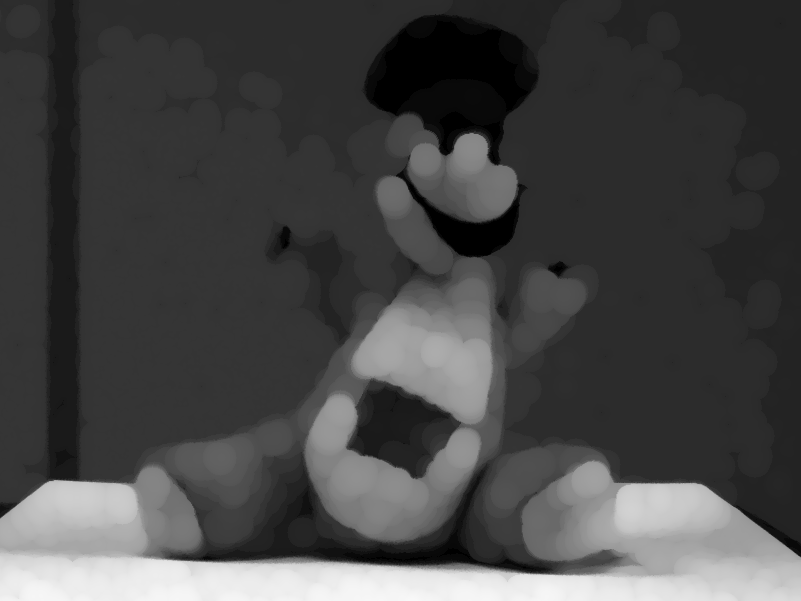}\\
	(c) Closing&
	(d) L-closing
	\end{tabular}
\caption{(a) Opening $(\gamma_b((f^d)^c))^c$ and (b) logarithmic-opening $(\gamma_b^{\protect \LP}((f^d)^c))^c$ on the dark image $f^d$. (c) Closing $(\varphi_b((f^d)^c))^c$ and (d) logarithmic-closing $(\varphi_b^{\protect \LP}((f^d)^c))^c$. The sf $b$ is an hemisphere of radius 15 pixels.}
\label{fig:comp_open_close_vs_Lopen_Lclose_Nessie}
\end{figure}

%
%

\section{Conclusion and perspectives}
\label{sec:concl}

Logarithmic-mathematical morphology is introduced in this paper. The fundamental operators of logarithmic-dilation $\delta_b^{\LP}$ and erosion $\varepsilon_b^{\LP}$ are defined for a structuring function thanks to the LIP-addition law $\LP$. Their expressions are related to the ``classical''  dilation $\delta_b$ and $\varepsilon_b$ facilitating their implementation. As both operators form an adjunction, a logarithmic-opening and closing are defined. The logarithmic-MM is compared to the ``classical'' MM based on an additive structuring function through several examples. Results show that Logarithmic-MM operators are particularly efficient to detect contrast variations in the dark parts (and also in the bright parts) of images caused by different illumination conditions. In future, logarithmic-MM operations will be extended to colour and multivariate images.

%
\bibliographystyle{splncs04}
\bibliography{refs}

\end{document}